\documentclass[journal]{IEEEtran}
\pdfoutput=1

\usepackage[numbers]{natbib}
\usepackage{multicol}
\usepackage{graphicx}
\usepackage{amsmath} 
\usepackage{amssymb}  
\usepackage{amsthm}
\usepackage{array}
\usepackage{algorithm}
\usepackage{algpseudocode}
\usepackage{color}

\newcommand{\jac}[2]{\frac{\partial {#1}}{\partial{#2}}}

\newtheorem{theorem}{Theorem}
\newtheorem{claim}{Claim}
\newtheorem{defn}{Definition}

\algtext*{EndWhile}
\algtext*{EndIf}
\algtext*{EndFor}

\graphicspath{{./}{figures/}}

\begin{document}

\title{Learning the Problem-Optimum Map: Analysis and Application to Global Optimization in Robotics}

\author{Kris~Hauser
\thanks{K. Hauser is with the Departments
of Electrical and Computer Engineering and of Mechanical Engineering and Materials Science, Duke University, Durham,
NC, 27708 USA e-mail: kris.hauser@duke.edu.}
}

\maketitle

\begin{abstract}
This paper describes a data-driven framework for approximate global optimization in which precomputed solutions to a sample of problems are retrieved and adapted during online use to solve novel problems.  This approach has promise for real-time applications in robotics, since it can produce near-globally optimal solutions orders of magnitude faster than standard methods.  This paper establishes theoretical conditions on how many and where samples are needed over the space of problems to achieve a given approximation quality.  The framework is applied to solve globally optimal collision-free inverse kinematics (IK) problems, wherein large solution databases are used to produce near-optimal solutions in sub-millisecond time on a standard PC. 
\end{abstract}

\section{Introduction}
Real-time optimization has been a longstanding challenge for robotics research due to the need for robots to react to changing environments, to respond naturally to humans, and to interact fluidly with other agents.  Applications in robotics are pervasive, including autonomous vehicle trajectory optimization, grasp optimization for mobile manipulators, footstep planning for legged robots, and generating safe motions in proximity to humans.  Global optimality has long been sought, but is generally computationally complex in the high-dimensional, nonconvex problems typical in robotics.  As a result, most researchers resort to local optimization or heuristic approaches, which have no performance guarantees. 

A promising approach is to integrate precomputed data (i.e., experience) into optimization to reduce online computation times~\cite{berenson2012robot,6672028,Goldfeder2011,jetchev2013fast,lien2009planning,PCA2014,PCCL2012}.  This idea is attractive because humans spend little time deliberating when they have seen a similar problem as one solved before, and hence robots may also benefit by learning from experience.  But because experience is time consuming to generate, it is important to pose the question, {\em how much data is needed to learn robot optimal control tasks}?

This paper formalizes this question in a fairly general context of global nonlinear optimization with nonlinear constraints. Consider a structure where problems are drawn from some family of related problems, such that problems can be parameterized by a set of continuous {\em P-parameters} that modify a problem's constraints and objective function.  These parameters
are not themselves decision variables, but rather they alter the constraints of the problem and hence affect the value of the optimal solution (Fig.~\ref{fig:framework}).  For example, in inverse kinematics, the P-parameter space is identical to the notion of task space.  For a robot picking and placing objects on a table, the P-parameters may be the source and target pose of the object.  In grasp optimization, P-parameters specify some parametric representation of object geometry.  In other words, they give a deterministic ``feature vector'' of the type of optimization problems we would like to learn.  We can then speak of learning a {\em problem-optimum map} (note that some problems may have multiple optima, so a map may be nonunique).

We present a quite general Learning Global Optima (LGO) framework in which a learner is given a database of optimal solutions, called {\em examples}, that are generated for problems sampled from a given problem distribution, assumed to be representative of problems faced in practice.  Since this step is offline, brute force methods can be used to obtain global optima or near-optima with high probability.  In the query phase, a solution to a previous problem is adapted to satisfy the constraints of a novel problem, typically via local optimization.

\begin{figure}[t]
\centering
\includegraphics[width=0.4\textwidth]{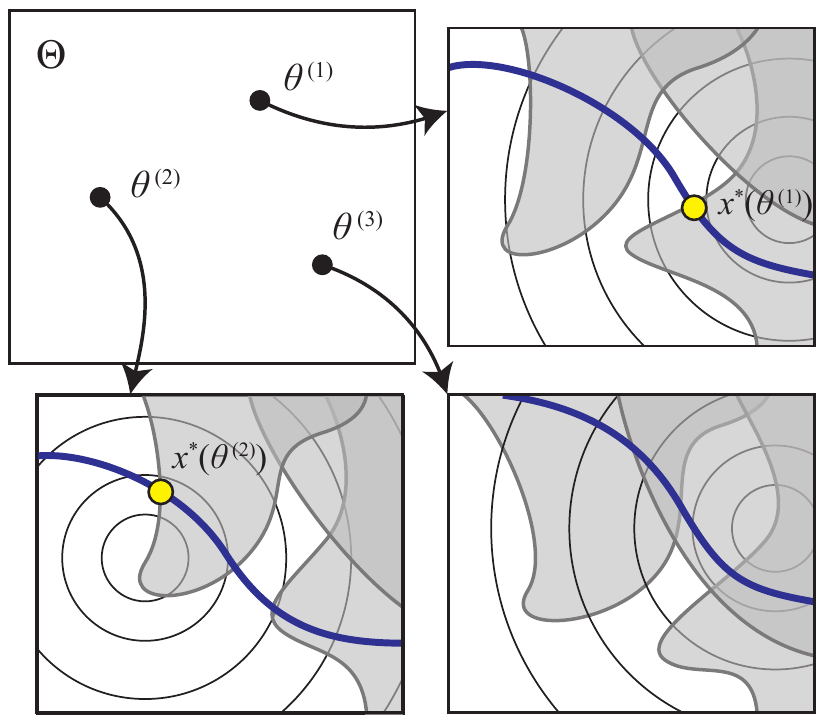}
\caption{Illustrating the problem space concept.  Instances $\theta^{(1)}$, $\theta^{(2)}$, and $\theta^{(3)}$ drawn from problem space $\Theta$ each correspond to different optimization problems.  Each problem has an objective function (level sets drawn as circles), inequality constraints (shaded region), and equality constraints (thick curve) dependent on the choice of $\theta$.  Each problem $\theta \in \Theta$ has a set of optima $x^\star(\theta)$, which may be empty if the problem has no feasible solution (e.g., $\theta^{(3)}$). }
\label{fig:framework}
\end{figure}

The learner's performance relies on whether a {\em solution similarity principle} holds across problem space: the optimal solution to a nearby example problem is likely to be close to the global optimum of the query problem, and hence using it to seed local optimization is likely to yield global solution. 
We derive formal results guaranteeing that the solution similarity principle partially holds, under quite general assumptions regarding the structure of problem space.  Given certain smoothness constraints, the optimal cost function has bounded derivatives almost everywhere across problem space, except certain discontinuous boundaries.  As a result, it is possible to generate a finite database of examples to cover problem space ``sufficiently well'' to achieve bounded suboptimality in the query phase.

To some extent these are positive results, but they can also be interpreted negatively.  First, analysis indicates that the number of examples needed scales exponentially in problem space dimension.  This indicates that learning requires vast experience in high-dimensional problem spaces, unless further strong assumptions are made.  Second, the problem-optimum map is in fact only piecewise-continuous, which is challenging to learn for parametric function approximators like neural networks or Gaussian process models~\cite{PCA2014}.  This justifies the choice in our implementation of a $k$-nearest-neighbor approach that is better suited to capturing these discontinuities.

An implementation of the LGO framework is demonstrated on the decades-old --- but still surprisingly challenging --- problem of inverse kinematics (IK).  An general solution would be able to calculate optimal solutions, handle redundant robots, gracefully handle infeasible problems, and incorporate collision avoidance, all done in real-time.  Yet current approaches fall short of addressing all of these challenges. 
Our proposed implementation addresses optimality, redundancy, infeasibility, and collision avoidance in a unified manner.  It works with general articulated robots and produces near-optimal, collision-free IK solutions typically in less than one millisecond on a standard PC.  It can also predict infeasible queries with high confidence (over 98\%), and is amenable to a ``lifelong learning'' concept that, given a handful of seed queries, automatically populates the IK database in the background.  As a result, the solver progressively solves IK queries faster and more reliably as more time is spent using it.  This IKDB package is made publicly available at \url{http://motion.pratt.duke.edu/ikdb/}.


\section{Related work}
The results of this work are relevant to the significant body of related literature in robot control learning.  Reinforcement learning proposes that robots record their experience interacting with the physical world and then progressively optimize their behavior to improve performance.  Other researchers study the approach of learning from demonstration, which asks human teachers to provide robots with instances of ``good'' motions.  In either case, it is quite time consuming to provide physical robots with experience or demonstrations from human teachers, making the question of ``how much experience is needed?'' of the utmost importance.

Despite decades of research, it is still challenging to quickly compute high quality solutions for high dimensional nonlinear optimization problems.  
A promising approach is to ``reuse'' previous solutions to solve similar problems, in which reuse is applied in a variety of forms.  Machine learning approaches are popular, and techniques have been used to predict the success rate of initial solutions or trajectories for optimization~\cite{Cassioli2010,PCA2014}.  
Other authors have used databases of grasps to generate grasps for novel objects~\cite{6672028,Goldfeder2011}. 
Another form of reuse is compositional, in which small solutions are composed to solve larger problems.  For example, footstep-based legged locomotion planners often store a small set of optimized motion primitives and replay them to execute footsteps~\cite{chestnutt2003planning,kuffner2001footstep}.  Rather than directly replaying motions, some other methods apply more sophistication adaptation methods to increase the range of problems to which a given example can be adapted~\cite{hauser06}.  Experience has also been studied for avoiding recomputation of roadmap structures when obstacles change~\cite{lien2009planning,PCCL2012}. 

The instance-based framework used in this paper is arguably most closely related to the work of Jetchev and Toussaint~\cite{jetchev2013fast}, which select past plans to initialize trajectory optimization.  Similarly, this paper uses experience to initialize optimization, but considers the more general context of nonlinear optimization problems, and focuses on the conditions under which such databases provide theoretical guarantees on solution quality.  Furthermore, the prior work uses locally optimal examples in the database, whereas our work uses globally optimal ones.

Our automated IK learning framework bears some similarity to the Lightning framework of Berenson et al~\cite{berenson2012robot} in which past plans are reused in parallel with planning from scratch.  In that work, past examples are retrieved based on nearest feature vector containing the start and end positions of the planning query, whereas our work is more general and accepts any P-parameters of the class of optimization problems.  Furthermore, Lightning uses past planning queries to generate new experience, while our framework uses a handful of user-driven queries to seed the database, but can then populate the database automatically thereafter. 

In the context of inverse kinematics, researchers from computer animation have used learning from human motion capture data~\cite{Grochow:2004:SIK:1186562.1015755,ho2013topology}.
In robotics, several authors have used machine learning methods to learn local inverse kinematics solutions for manipulator control in Cartesian space~\cite{souza2001learning,wang1999lagrangian}.  Some researchers are also able to incorporate obstacles~\cite{mao1997obstacle,zhang2004obstacle}.  By contrast, our work considers the global IK problem, which is more appropriate for determining goal configurations for motion and grasp planning. Perhaps the closest attempt to ours to address this problem is an approach that tries to discover structure in the IK problem space via clustering~\cite{demers1991learning}. We explore more fully some of the preliminary ideas in that work.



\section{Learning Global Optima framework}
\label{sec:statement} 

First, let us summarize the intuition behind our main results.  Consider a family of decision problems whose constraints and objectives themselves vary according to external parameters, such as the initial state of a trajectory, obstacle positions, target states, etc.  Since these parameters are not under the direct control of the robot, they are not {\em decision parameters} (denoted in this paper as $x$) under the control of the agent, but rather {\em P-parameters} (denoted in this paper as $\theta$).  Since the optimal solution $x^\star$ depends on the current values of the external parameters, such as a changing environmental obstacle causing the robot to take a different path, there is a relationship between P-parameters and optimal solutions. In essence, a P-parameter vector $\theta$ can be thought of {\em producing} an optimization problem that the robot must then solve to generate $x^\star$.  

In the context of robotics, P-parameters can include the robot's initial state, the query (e.g., a target location), and characteristics of the environment (e.g., the positions of obstacles).  The problem's decision variables may be diverse, such as a robot configuration, a representation of a trajectory to be executed, or the parameters of a feedback control policy.  

We intend to study the relationship between $\theta$ and $x^\star$ across the space of P-parameters, which is a set we call {\em problem space}. The concept of problem space is related to the notion of {\em task space} in IK or operational space control, and {\em feature space} in machine learning.  A problem space gives the range of possible problems encountered in practice, and we are interested in performing well across this entire space.  (In practice, it may be more appropriate to speak of a distribution over problems, but in this case we consider the range of likely problems.)

\subsection{Illustrative examples}

As an example, let us consider a family of IK problems. {\em Note that for the next two paragraphs we shall adopt traditional IK notation, which conflicts with the notation used elsewhere throughout this paper.}  Suppose a robot is asked to reach the IK objective $x_d = x(q)$, where $x(q)$ is the forward dynamics mapping from configurations $q$ to tasks $x$.  Here, we have $q$ being the decision variable which the IK solver is meant to solve, and $x_d - x(q) = 0$ the constraint. If we are now to think of the robot being asked to reach a range of IK objectives, we now have $\theta \equiv x_d$ being the P-parameter of this problem. If we are now to let $x_d$ range over the space of all possible or all likely tasks, this range of constraints gives the problem space $\Theta$. 

If we now have the robot required to simultaneously reach $x_d$ while avoiding an obstacle at position $p$.  We can encode this constraint by a distance inequality $d(q,p) \geq 0$ that measures how far the robot is at configuration $q$ from the object at position $p$.  Now, if $p$ were to also vary, we should add it to the P-parameter vector to obtain $\theta = (x_d,p)$.  If, on the other hand, only the $z$-coordinate of $p$ were to vary, we should only set $\theta = (x_d,p_z)$, treating the $x$- and $y$-coordinates of $p$ as constants in the constraint $d(q,p) \geq 0$.

\subsection{Summary of results}

\begin{figure}
\centering
\includegraphics[width=0.47\textwidth]{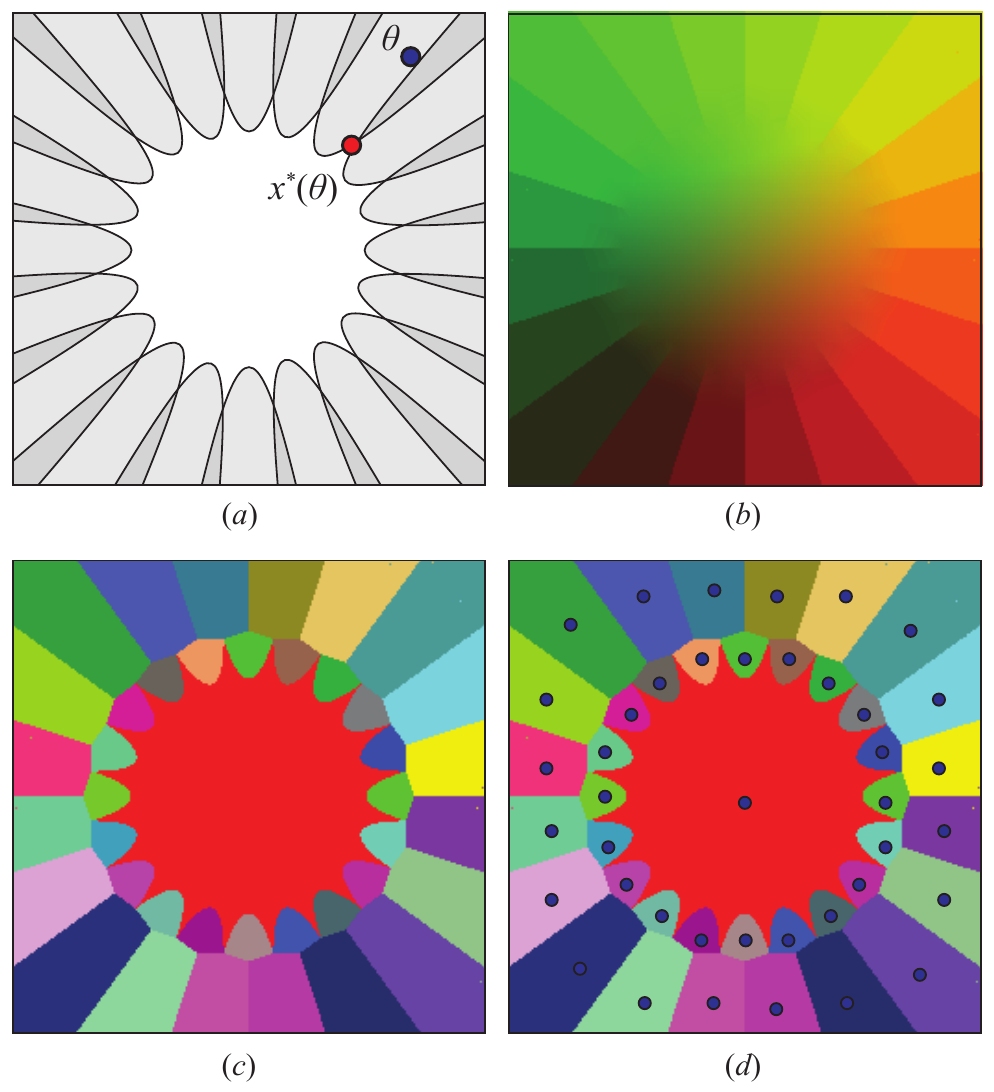}
\caption{(a) A simple 2D problem family in which the goal is to minimize the distance between the optimization parameter $x$ and the P-parameter $\theta$ subject to concave polynomial constraints. (b) The optimal solution function, with $x_1$ and $x_2$ plotted as red and green intensities, respectively.  It is piecewise smooth (best viewed in color).  (c) The problem space partitioned into regions of unique active set. (d) An ideal learner needs only to store one example in each region (circles).  It will globally solve a new problem using local optimization, starting from the example in the same region as that problem.}
\label{fig:StarProblem}
\end{figure}

We now state the intuition behind our main results, and illustrate them on the simple example of Fig.~\ref{fig:StarProblem}(a):
\begin{itemize}
\item If the mapping from P-parameters to constraints is continuous, then the mapping from $\theta$ to optimal solutions $x^\star(\theta)$ is a piecewise continuous mapping (Fig.~\ref{fig:StarProblem}(b)).  (Note that it is only a partial mapping if some $\theta$ define infeasible optimization problems.)
\item $x^\star(\theta)$ is continuous over connected regions of problem space in which the active set is unique and equal (Fig.~\ref{fig:StarProblem}(c)).  Moreover, it satisfies a Lipschitz bound if the active set Jacobian is nonsingular.
\item Let a {\em learner} consist of a database of examples and a subroutine for subselecting and adapting examples to a novel problem. We define a notion of ``goodness'' of a learner as the worst-case suboptimality of the adapted solution to a randomly drawn problem, relative to its true optimum.  Optimal goodness is achieved if the learner generates $x^\star(\theta)$ everywhere.
\item An idealized learner achieves optimal goodness if: 1) it stores an example in each contiguous region of problem space in which the active set is constant, 2) for a new problem specified by a P-parameter vector $\theta$, it retrieves an example in the same region as $\theta$, and 3) it uses local optimization, seeded from the example's solution and warm-started from its active set (Fig.~\ref{fig:StarProblem}(d)).
\item The worst-case number of examples for an ideal learner is the number of possible combinations of active constraints, which is at worst case exponential in problem size.
\end{itemize}
The remainder of this section and the next will formalize this intuition.

\subsection{Mathematical statement}
\label{sec:mathematical framework} 

In the proposed framework, we are given a class of optimization problems over an $n$-dimensional decision parameter $x$:
\begin{equation}
\begin{gathered}
\text{Given functions $f$, $g$, and $h$}\\
\text{Find $x \in \mathbb{R}^n$  s.t.}\\
f(x) \text{ is minimized,}\\
g(x) \leq 0,\text{ and} \\
h(x) = 0.
\end{gathered}
\label{eq:BasicOptimization}
\end{equation}
Here, the functions are defined with ranges $f : \mathbb{R}^n \rightarrow \mathbb{R}$, $g : \mathbb{R}^n \rightarrow \mathbb{R}^m$, and $h : \mathbb{R}^n \rightarrow \mathbb{R}^p$, and the equality and inequality are taken element-wise.  The functions $f$, $g$, and $h$ are in general nonconvex.  However, the feasible set and the objective function are required to be bounded.  The goal of global optimization is to produce an optimal solution $x$ if one exists, or to report infeasibility if none exists. 

A {\em problem space} is defined by variations of $f$, $g$, and $h$.  In this paper, variations in problems will be encapsulated in the form of a P-parameter variable $\theta$.  Specifically, the functions $f$, $g$, and $h$ will be themselves function of the parameter $\theta \in \Theta$, where $\Theta \subseteq \mathbb{R}^r$ is known as the problem space.  To make the dependence on $\theta$ explicit, we write $f(x) \equiv f(x,\theta)$, $g(x) \equiv g(x,\theta)$, and $h(x) \equiv h(x,\theta)$. (Here it is important that $f$, $g$, and $h$ are fixed functions of both $x$ and $\theta$ and not any external factors --- in other words, $\theta$ must capture all variability between problems.) We can now rewrite \eqref{eq:BasicOptimization} as:
\begin{equation}
\begin{gathered}
\text{Given P-parameter $\theta \in \Theta$,}\\
\text{Find $x \in \mathbb{R}^n$  s.t.}\\
f(x,\theta) \text{ is minimized over $x$},\\
g(x,\theta) \leq 0,\text{ and} \\
h(x,\theta) = 0.
\end{gathered}
\label{eq:CoparameterOptimization}
\end{equation}
With this definition, it is apparent that the only external parameter affecting the solution is the P-parameter $\theta$.  It is now possible to speak of the set of optimal solutions $x^\star(\theta)$, which is a deterministic function of $\theta$.  $x^\star(\theta)$ may be empty if there is no solution.  If it is nonempty, an optimal cost $f^\star(\theta)$ exists.  We refer to the problem of computing one such solution, or determining that none exists, as the problem $P(\theta)$.

As a final note, in most of this paper, we shall treat $m$, $n$, $p$, and $r$ as fixed.  In practice, however, it may be possible to share examples across problems of differing dimension.  For example, solutions to motion planning problems may be represented as a variable number of waypoints, with a set of decision variables and constraints for each waypoint.  The algorithm presented below still applies, but the analysis would require additional assumptions about problem structure. 

\subsection{Database computation phase}

We assume the learner has access to compute a database of problem-solution pairs, each of which is known as an {\em example}.  Specifically, the database consists of $D = \{(\theta^{(i)},x^{(i)})\quad |\quad i=1,\ldots,N\}$ where $\theta^{(i)}$ is a P-parameter sampled from $\Theta$, and $x^{(i)}$ is an optimal solution in $x^\star(\theta^{(i)})$.  Later we will use {\em nil} to mark that no solution was found, but for now let us assume a database entirely consisting of successfully solved examples.  The distribution of $\theta^{(i)}$ should be relatively diverse, and should approximate the distribution of problems expected to be encountered in practice. The solutions are computed by some global optimization method during an extensive precomputation phase.

In practice, for generating example solutions, we resort to the use of metaheuristic global optimization techniques (e.g., random restarts, simulated annealing, or evolutionary algorithms).  However, their approximate nature means that it cannot be guaranteed that a precomputed solution is truly optimal, or that a precomputed failure is truly a failure.  As a result the database will contain some noise.  We will discuss the effects of noise, and methods for combating it, in subsequent sections.

\subsection{Query phase}
\label{sec:QueryPhase}

The query phase attempts to answer a novel query problem specified by P-parameter $\theta^\prime$.  We will analyze a learner that has access to two auxiliary functions:
\begin{enumerate}
\item Let $S(\theta)$ be a {\em selection function} that produces a subset of $k$ examples in $D$ for any problem $\theta \in\Theta$.  Usually these are assumed to be close in problem space to $\theta^\prime$, such as by finding its $k$-nearest neighbors given a distance metric $d(\theta,\theta^\prime)$ in $\Theta$ space. 
\item Let $\mathcal{A}(x,\theta,\theta^\prime)$ be the {\em adaptation function} that adapts an optimal solution $x$ for problem $\theta$ to another problem $\theta^\prime$.  The result of $\mathcal{A}$ is either a solution $x^\prime$ or {\em nil}, which indicates failure.
\end{enumerate}
The learner is assumed to proceed as follows:
\begin{itemize}
\item {\em Retrieval:} Select the problems $\theta^{(i_1)},\ldots, \theta^{(i_k)} \gets S(\theta^\prime)$
\item {\em Adaptation:} For each $j=1,\ldots,k$, run $\mathcal{A}(x^{(i_j)},\theta,\theta^\prime)$ to locally optimize the solution for the selected problems.  Return the best locally optimized solution, or {\em nil} if none can be found.
\end{itemize}
It may be possible to learn a different representation of the map from $\theta$ to $x$, e.g. by using neural networks or Gaussian processes.  However, $k$-NN more readily admits more formal analysis, and also has some benefits handling  discontinuities as we will show in the next section.

\section{Theoretical study}

The core assumption behind the learner is that two optimization problems with similar P-parameters will have similar solutions, and therefore the adaptation of the solution for one problem to the other is very likely to be successful.  So, if the database contains problems sampled sufficiently densely from $\Theta$, then any novel problem in $\Theta$ should be solvable by adapting one of those prior solutions.  This section justifies the argument formally, and yields some intuition about which regions of $\Theta$ are well-behaved and can be sampled sparsely, and which regions must be sampled more densely.

\subsection{Library goodness}

We first set out to formalize the concept of {\em good database}. Let us assume the database $D = \{(\theta^{(i)},x^{(i)})\quad |\quad i=1,\ldots,N\}$ consists entirely of exact, optimal solutions where $\theta^{(i)}$ is a P-parameter from $\Theta$, and $x^{(i)}$ is an optimal solution for $P(\theta^{(i)})$.  Let us also define $\sigma$ as a volumetric measure over $\Theta$.  Ideally, $\sigma$ should capture how likely a problem is encountered during online use.  

We will first need a preliminary definition.  Let $\alpha \geq 0$ be an arbitrary parameter. 
\begin{defn}
A problem $\theta^\prime \in \Theta $ is $(\alpha,S)$-{\em adaptable} iff there exists an example $(\theta,x) \in S(\theta^\prime)$ such that $\mathcal{A}(x,\theta,\theta^\prime) \neq nil$ and $f(\mathcal{A}(x,\theta,\theta^\prime),\theta^\prime) \le \alpha + f(x^\prime,\theta^\prime)$, where $x^\prime$ is an optimal solution for $P(\theta^\prime)$
\end{defn}
In other words, some primitive in $S(\theta^\prime)$ is able to solve the problem ``well'' --- up to the tolerance $\alpha$ --- using the adaptation function $\mathcal{A}$.  As $\alpha$ shrinks toward 0, the condition of $\alpha$-goodness requires that the problem be solved with quality increasingly approaching optimal.  We are now ready to state the main definition.
\begin{defn}
$D$ is an $(\alpha,\beta,S)${\em -good library} if the set of $(\alpha,S)$-adaptable problems in $\Theta$ has volume at least $1-\beta$ times the volume of $\Theta$.  Specifically, it requires that $\sigma(F) \geq (1-\beta) \cdot \sigma(\Theta)$ with $F = \{ \theta \in \Theta \quad | \quad \theta\text{ is }(\alpha,S)\text{-adaptable} \}$.  
\end{defn}
The parameter $\beta$ controls how many problems satisfy the $(\alpha,S)$-adaptability condition.  The implication is that when $\beta$ shrinks towards 0, the fraction of failed or poor-quality adaptations shrinks toward 0.  If $S(\theta)$ produces large subsets, the likelihood of covering a large region of $\Theta$ with $\alpha$ quality increases, but computational cost also increases.  The goal of our technique is to produce $D$ and $S$ such that $D$ is $(\alpha,\beta,S)$-good for relatively low values of $\alpha$ (indicating high quality) and $\beta$ (indicating high success rate) and small $|S(\theta)|$ (indicating low computational cost).

\subsection{Smoothness of the optimal objective, almost everywhere}
\label{sec:Smoothness}
We begin by examining the question of whether two optimization problems that are similar --- in the sense of having two similar P-parameters --- also have solutions of similar quality.  It will set us up in the following section to examine whether a solution to one problem can be locally optimized to optimal or near-optimal quality. This section makes heavy use of results in sensitivity analysis of the Karush-Kuhn-Tucker (KKT) conditions~\cite{Jongen1990,shapiro1985second}. 

Note, first, that optimal solutions to similar problems are {\em themselves} not necessarily similar even if the optimal cost is similar.  An example is given in Fig.~\ref{fig:OptimalToyExample}.  Even though the optimal cost function $f(\theta)$ is  continuous, the solution $x^\star(\theta)$ is not.  Note that in general the optimal cost may be discontinuous as well.

\begin{figure}
\centering
\includegraphics[width=0.45\textwidth]{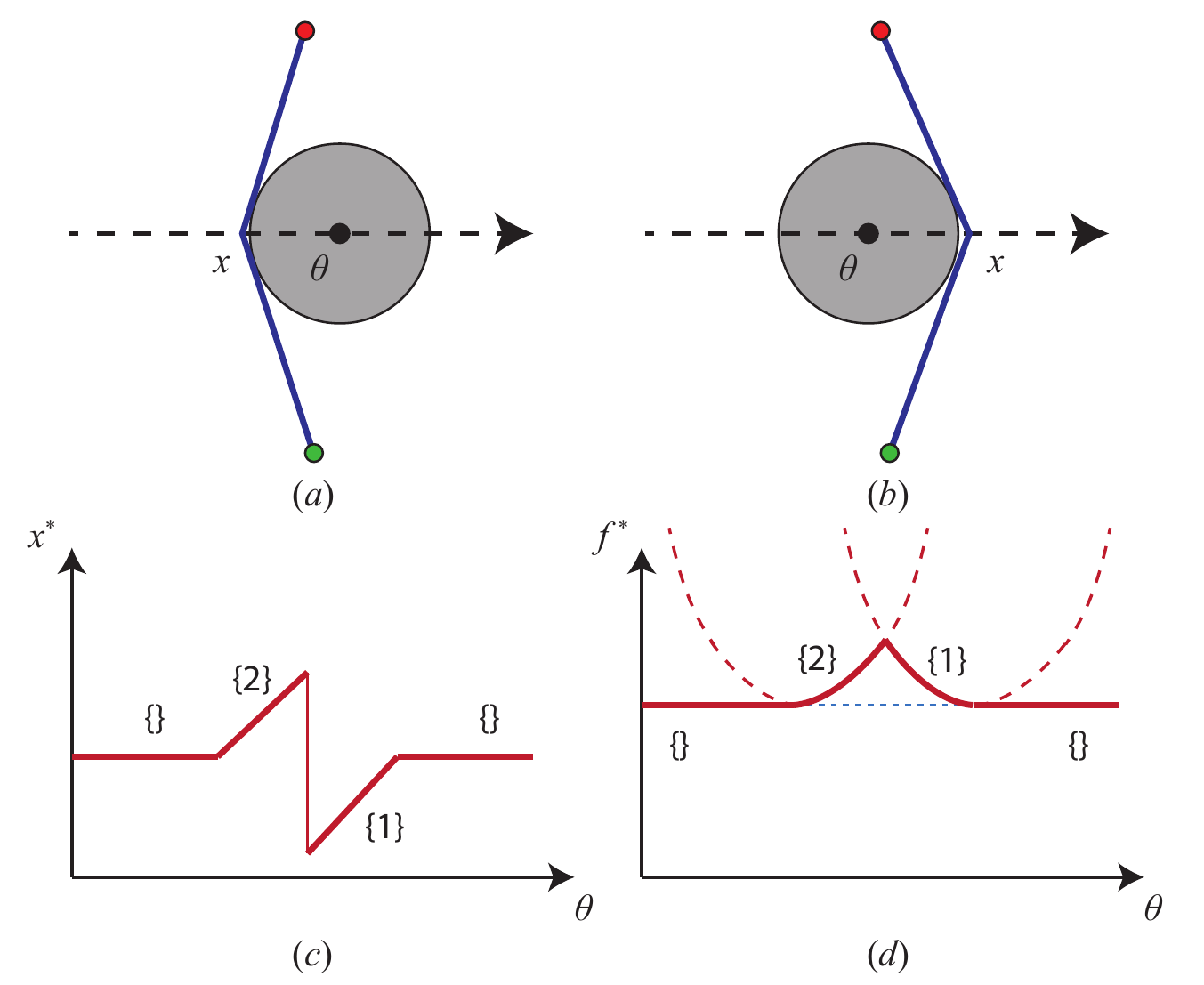}
\caption{A simple problem class where two points are to be connected with a shortest collision-free path.  The $x$ parameter controls the horizontal position of the path's midpoint while the $\theta$ P-parameter controls the horizontal position of the circular obstacle.  In (a) the optimal path is to the right of the obstacle, but if the obstacle were to move slightly to the position in (b) it moves to the left.  (c) The optimal solution $x^\star(\theta)$ is a deterministic but discontinuous function of $\theta$. (d) The optimal cost is a continuous function of $\theta$.  Dashed lines indicate local optima.}
\label{fig:OptimalToyExample}
\end{figure}

First, we will prove the following:
\begin{theorem}
\label{theorem:OptimalDerivative}
The optimal cost function $f^\star(\theta)$ has a defined derivative under the assumptions that $P(\theta)$ is feasible with a unique optimum $x$ and unique active set, and $f$, $g$, and $h$ have bounded first derivatives with respect to both $x$ and $\theta$. 
\end{theorem}
\begin{proof}
The KKT conditions must be satisfied at an optimum $x$.  That is, there must exist KKT multipliers $\lambda$ and $\mu$ such that
\begin{equation}
\begin{gathered}
\jac{f}{x}(x,\theta)^T + \jac{g}{x}(x,\theta)^T \mu + \jac{h}{x}(x,\theta)^T \lambda = 0 \\
g(x,\theta) \geq 0 \\
h(x,\theta) = 0 \\
\mu^T g(x,\theta) = 0 \\
\mu \geq 0.
\end{gathered}
\end{equation}
The set of indices for which the multipliers $\mu_i$ are nonzero are known as an {\em active} set.  For each index $i$ in an active set $A$, $g_i(x,\theta)$ is identically 0.  As a result, if we were to identify an active set, we can rewrite the KKT conditions as:
\begin{equation}
\begin{gathered}
\jac{f}{x}(x,\theta)^T + \jac{g_A}{x}(x,\theta)^T \mu_A + \jac{h}{x}(x,\theta)^T \lambda = 0 \\
g_A(x,\theta) = 0, \quad g_{\bar{A}}(x,\theta) \geq 0 \\
h(x,\theta) = 0 \\
\mu_{\bar{A}} = 0, \quad \mu_{A} > 0 
\end{gathered}
\end{equation}
where the subcript $A$ denotes extraction of rows in the active set, and the subscript $\bar{A}$ denotes extraction of rows in the inactive set.  It must be noted that there are some cases in which the active set is non-unique, and these cases may cause discontinuities in $f^\star(\theta)$.

Let us consider the case where $A$ is the unique active set.  If we lump $\mu_A$ and $\lambda$ into a multiplier vector $z$ and stack $g_A$ and $h$ into a vector field $j=\left[\begin{array}{c} g_A \\ h \end{array}\right]$, then the solution and KKT multipliers must satisfy the conditions
\begin{equation}
\begin{gathered}
\jac{f}{x}(x,\theta)^T + \jac{j}{x}(x,\theta)^T z = 0\\
j(x,\theta) = 0.
\end{gathered}
\label{eq:ActiveSetKKT}
\end{equation}

Let us consider an infinitesimally small shift of problem to $\theta^\prime = \theta + \Delta \theta$.  We examine the direction $\Delta x$ in which the optimal solution changes to $x^\prime = x + \Delta x$ while keeping the active set constant.  The optimal cost will shift by $\Delta f = f(x^\prime,\theta^\prime) - f(x,\theta) = \jac{f}{x} \Delta x + \jac{f}{\theta} \Delta \theta$ where the last term uses the first-order Taylor expansion and we have left the dependence on $x$ and $\theta$ implicit.

Replacing $\jac{f}{x}$ with the first equality in the active set KKT conditions \eqref{eq:ActiveSetKKT}, we get
\begin{equation}
\Delta f = - z^T \jac{j}{x} \Delta x + \jac{f}{\theta}\Delta \theta.
\label{eq:DeltaFExpansion}
\end{equation}
Applying Taylor expansion to the second KKT equality, we get
\begin{equation}
j + \jac{j}{x} \Delta x + \jac{j}{\theta}\Delta \theta = 0 
\end{equation}
which can be rewritten
\begin{equation}
\jac{j}{x} \Delta x  = - \jac{j}{\theta}\Delta \theta
\end{equation}
because $j =0 $ at the optimum. Finally, replacing this equation in \eqref{eq:DeltaFExpansion} we get
\begin{equation}
\Delta f = \left(\jac{f}{\theta} + z^T \jac{j}{\theta}\right)\Delta \theta.
\end{equation}
Since each of the derivatives in the above equation are bounded and $z$ is finite, $f^\star(\theta)$ is continuously differentiable with derivative $\jac{f^\star}{\theta} = \jac{f}{\theta} + z^T \jac{j}{\theta}$. 
\end{proof}

From \eqref{eq:ActiveSetKKT}, we can also express $z = -(\jac{j}{x} \jac{j}{x}^T)^{-1} \jac{j}{x} \jac{f}{x}^T$, where the inverse exists under the assumption that a unique solution exists for $x$.  The matrix $(\jac{j}{x} \jac{j}{x}^T)^{-1} \jac{j}{x}$ is actually the transpose of the pseudoinverse of $\jac{j}{x}$, and as a result the rate of change of $f^\star(\theta)$ is bounded when $\jac{j}{x}$ is far from singular. In the context of inverse kinematics, this occurs when the robot is in a high manipulability configuration.

This reasoning may be extended to two similar claims, which have somewhat more involved proofs and hence the details are omitted for the sake of brevity.
\begin{claim}
\label{claim:Lipschitz}
For {\em nearly all} $\theta$ (specifically, a subset of $\Theta$ with measure $\sigma(\Theta)$), there exists a neighborhood $\mathcal{N}(\theta,R)$ for some $R > 0$ over which the optimal cost function satisfies a Lipschitz bound.  The Lipschitz constant is a linear function of the inverse of the smallest singular value of the active set Jacobian over points in $\mathcal{N}(\theta,R)$. 
\end{claim}
The claim requires that $f$ and $j$ also satisfy a Lipschitz bound, and that there is no discontinuity in $f^\star$ across the neighborhood (this excludes points $\theta$ of discontinuity, which form a set of measure 0).  The proof shows that the maximum change of $f^\star$ over the neighborhood is at most $R (L_{f,\theta} + L_{f,x} Q L_{j,\theta})$, where $L_{f,\theta}$, $L_{f,x}$, and $L_{j,\theta}$ are Lipschitz constants bounding the rate of change of $f$ with respect to theta, $f$ with respect to $x$, and $g$ and $h$ with respect to $\theta$.  The parameter $Q$, is a Lipschitz constant upper bounding the norm of the pseudoinverse of the active set Jacobian with respect to $x$ across the entire neighborhood (i.e., the inverse of the the minimum singular value of $\jac{j}{x}$).

The second claim relaxes the requirement of unique active sets for certain problems with ``tame'' nonunique active sets. 
\begin{claim}
\label{claim:DiscontinuousDerivatives}
If at $\theta$ the active set is nonunique, and 1) the active set Jacobian is nonsingular for each active set, and 2) the optimal solution $x$ is the same under each active set, then $f$ satisfies a Lipschitz condition.
\end{claim}
These conditions are most often observed in the addition or removal of an inequality constraint, as seen in the first and third point of derivative discontinuity in Fig.~\ref{fig:OptimalToyExample}: the derivatives leading in and leading out of the discontinuity are both bounded, so the point still satisfies a Lipschitz condition.  

\subsection{Theoretical and practical goodness conditions}

Now let us turn to the original question of whether an example $(\theta,x)$ in the database may be adapted with high quality to a novel problem $P(\theta^\prime)$.  We impose one major requirement on the adaptation function $\mathcal{A}$: if $P(\theta)$ has a unique optimum $x$ with a certain active set, then $\mathcal{A}(x,\theta,\theta^\prime)$ will yield a solution at least as good as the solution to the optimization problem $P(\theta^\prime)$ {\em keeping the active set constant}.  As a consequence, if the optimal solution to $P(\theta^\prime)$ shares the same active set, then $\mathcal{A}(x,\theta,\theta^\prime)$ will yield an optimal solution.

Let us define a function $\mathcal{O}(\theta)$ that produces a set of all active sets that define optima of $P(\theta)$.  For infeasible problems, we can take the convention that $\mathcal{O}(\theta)=\{\}$.  If the problem statement \eqref{eq:CoparameterOptimization} is nondegenerate, then across most of the space --- that is, all of $\Theta$ except for a set of measure 0 ---  $\mathcal{O}(\theta)$ contains one active set.  We can now imagine $\Theta$ partitioned into equivalence classes $C_1,\ldots,C_M$ by connected regions with uniform $\mathcal{O}(\theta)$.  In the nondegenerate case, every region with nonzero volume corresponds to a unique active set. 

We can conceive of a hypothetical ``omniscient'' selection function that always yields an example in the same equivalence class as $\theta^\prime$ whenever $\theta^\prime$ has a unique solution.  (Of course, it is in general impossible to determine the active set without solving $P(\theta^\prime)$ in the first place.)  Hence, we claim that one could attain a perfect, $(0,0,S)$-good database under the conditions:
\begin{enumerate}
\item The problem statement \eqref{eq:CoparameterOptimization} is nondegenerate,
\item The database $D$ has at least one example in each equivalence class with nonzero volume,
\item And the selection function $S$ is ``omniscient,''
\end{enumerate}
Then for almost all $\theta^\prime$ except for a set of measure 0, $\theta^\prime$ is $(0,S)$-adaptable.  

Realistically, it is not possible to derive an omniscient selection function.  But, $k$-NN becomes an increasingly good approximation as the sampling of problem space grows denser and $k$ grows larger.  In the case where $S$ fails to produce a correct example everywhere, there may exist a relaxed $\beta$ such that the database is $(0,\beta,S)$-good.

Another issue is that the number of equivalence classes is potentially double exponential $2^{2^{K(m,n-p)}}$, where $K(a,b) = \sum_{i=0}^{b} {a \choose i}$.
If $n-p \geq m$  then $K(m,n-p) = 2^m$, but we can certainly say $K(a,b) \geq {a \choose \min(b,a/2)}$.  Moreover, since the constraints are nonlinear the number of connected partitions may be yet even higher.  With a smaller database than this, adaptations from one active set to another will incur some suboptimality.  It is then a question of what relaxed value of $\alpha$ we can expect to attain $(\alpha,\beta,S)$-goodness.  As a result of Claim~\ref{claim:Lipschitz}, if we restrict ourselves to problems in the ``well-behaved'' realm where the constraint Jacobian is far from singular, then a database with dispersion at most $R \leq \alpha / L$ will be $(\alpha,S)$-adaptable, where $L$ is the maximum Lipschitz coefficient in Claim~\ref{claim:Lipschitz} over the problems in the database.  Since each of these balls has volume $c_n R^n$ for constant $c_n = \frac{\pi^{n/2}}{\Gamma(n/2+1)}$, it is required that the database contain at least $\sigma(\Theta) L^n / (c_n \alpha^n)$ examples.

Hence in practice, to obtain an $(\alpha,\beta,S)$-good database for P-parameter spaces with large number of P-parameters $n$, we must include a vast number of examples, or be willing to relax $\alpha$ and $\beta$.  Moreover, it suggests that to obtain adaptations of uniformly high quality, the database should be sampled more densely in the regions where the active set Jacobian is near singular.

\section{Query implementation}

The basic LGO framework works fairly well as presented in Section~\ref{sec:QueryPhase}.  However, several parameters affect query performance in practice:
\begin{enumerate}
\item {\em Problem similarity metric.} Although similar problems have similar optimal solutions, optimal solution quality usually changes anisotropically in problem space as described in Sec.~\ref{sec:Smoothness}.  As a result, non-Euclidean distance metrics may be better for retrieving good examples. 
\item {\em Search strategy.} A brute-force NN search has $O(n)$ computational complexity, which becomes slow in large databases.  The use of fast NN data structures allows our method to scale to larger databases.  We use a ball tree in our implementation.  
\item {\em Number of neighbors} $k$ affects the robustness of the query phase.  Larger values of $k$ help combat the effects of limited database size and noise by allowing additional attempts at solving the problem.  This comes at the expense of greater computational cost for challenging queries. 
\item {\em Local Optimization Strategy.}  A straightforward approach applies a nonlinear optimizer like sequential quadratic programming (SQP).  But we consider faster approximate strategies below.
\item {\em Perturbations in local optimization.}  If feasibility is affected by variables that are not captured in the P-parameters, such as environmental obstacles, it is usually prudent to perturb the retrieved seeds before local optimization.  This is also helpful to address non-differentiable inequality constraints.
\end{enumerate}

We present a faster implementation, LGO-quick-query, in which we modify the adaptation step in two ways.  First, we sort the $x^{(i)}$'s in order of increasing distance $d(\theta^{(i)},\theta)$, and stop when the first solution is found.  If no solutions are found, failure is returned.  This allows the method to scale better to large $k$, since easy problems will be solved in the first few iterations.  Second, rather than using full local optimization, we simply project solutions onto the equality constraints $h(x,\theta)=0$ and those active inequalities that are met exactly at the prior example.  This method is often an order of magnitude faster than SQP, and provides sufficiently near-optimal feasible solutions.  

Specifically, given a prior example $(x^{(i)},\theta^{(i)})$ we detect the set $A$ of constraints in $P(\theta^{(i)})$ active at $x^{(i)}$, and then formulate the active set equality $g_A(x,\theta)=0$, $h(x,\theta)=0$ for the new problem. Starting at $x^{(i)}$, we then solve for a root $j(x^\prime,\theta^\prime) = 0$ via the Newton-Raphson method.  Bound constraints, e.g., joint limits, are also efficiently incorporated.  Because Newton-Raphson takes steps with least squared norm, it approximately minimizes $\|x - x^\prime\|^2$.  Although locally optimizing the objective function $f$ may produce better solutions, the strategy of keeping the solution close to the start will still retain the asymptotic goodness properties of the database.  

Pseudocode for this technique is as follows:

\begin{algorithm}
\caption{LGO-quick-query$(\theta,D)$}
\label{alg:QuickQuery}
\begin{algorithmic}[1]
\State Find the $k$-nearest neighbors $\theta^{(i_1)},\ldots,\theta^{(i_k)}$ in $D$, sorted in order of increasing $d(\theta^{(i_j)},\theta)$
\For{$j=1,...,k$} 
  \State $A \gets $ActiveSet$(x^{(i_j)},\theta^{(i_j)})$
  \State Simultaneously solve $g_A(x,\theta)=0$, $h(x,\theta)=0$ using \par
        \hskip\algorithmicindent the Newton-Raphson method, initialized at $x^{(i_j)}$ 
  \If{successful} \Return the local optimum $x$  \EndIf
\EndFor
\State \Return $nil$
\end{algorithmic}
\end{algorithm}

\subsection{Handling infeasible problems}

When the problem space contains many infeasible problems that may be drawn in practice, queries for infeasible problems are expensive because they always performs $k$ failed local optimizations per query.  In some cases it may be preferable to quickly terminate on infeasible problems.  LGO can be easily adapted to predict infeasible problems and avoid expending computational effort on them. 

We permit the database $D$ to contain infeasible problems, whose solutions are marked as $nil$.  If a large fraction of retrieved solutions for a query problem $\theta$ are $nil$, then it is likely that $\theta$ is infeasible as well. More formally, if $k'$ denotes the number of feasible examples out of the retrieved set, we use a confidence score PFeasible$(k')$ that determines how likely the problem is to be feasible given $k'$.  If PFeasible$(k')$ falls below a threshold, then we predict $\theta$ as being infeasible and do not expend further effort.  Otherwise, it is likely to be feasible but the database does not have sufficient coverage.  In this case we fall back to global optimization. 

\begin{algorithm}
\caption{LGO-query-infeasible$(\theta,D)$}
\label{alg:QueryInfeasible}
\begin{algorithmic}[1]
\State Find the $k$-nearest neighbors $\theta^{(i_1)},\ldots,\theta^{(i_k)}$ in $D$, sorted in order of increasing $d(\theta^{(i_j)},\theta)$.
\For{$j=1,...,k$}
  \If {$x^{(i_j)} \neq nil$} 
    \State Solve $g_A(x,\theta)=0$, $h(x,\theta)=0$ starting from $x^{(i_j)}$
    \If {successful}  \Return the local optimum $x$ \EndIf
  \EndIf
\EndFor
\State $k' \gets |\{j \in \{1,2,.\ldots,k\} \quad | \quad x^{(i_j)} \neq nil \}|$
\If {PFeasible$(k') > \tau$} \Return Global-Optimization($\theta$)
\EndIf
\State \Return $nil$
\end{algorithmic}
\end{algorithm}

To generate PFeasible, we use leave-one-out (l.o.o.) cross validation to estimate the empirical probability that the problem is infeasible given that $k'$ out of $k$ nearest neighbors are feasible.  

The {\em feasibility confidence threshold} $\tau$ should be chosen to trade off against the competing demands of average query time and incorrect predictions of problem infeasibility.  A value $\tau=1$ will never fall back to global optimization, while $\tau=0$ always falls back.  A high value leads to more consistent running times but slightly lower success rate.

\subsection{Database self-population and lifelong learning}

Our implementation also contains a ``lifelong learning'' mode that allows the database to self-populate given a handful of example problems.  A separate background thread generates more examples to the database without the user making queries.  To do so, it explores the problem space by random sampling in an automatically-determined, growing range of P-parameters, and running the global optimizer.  The background thread is limited to some maximum CPU load (30\% in our implementation). 

An important problem in lifelong learning is to determine when to stop adding problems to the database, as well as choosing which problems should be retained.  First, a size limit on the database is provided, currently set to 10,000,000 examples.  Second, a new problem is not added if adaptation from a prior example is successful and does not lose a certain amount of quality (determined by a threshold $\alpha$).  Finally, we have a low threshold $\tau_2$ (by default 0.02) whereby a new problem is not added if it has been determined that at least $\tau_2$ fraction of examples with exactly $k'$ out of $k$ feasible neighbors are feasible.  Typically these rejections only occurs with $k'=0$ or other small number, and once the database has grown to appreciable size.  We do not ``forget'' existing problems to make room for new ones, but that would be a useful feature for future work.

To communicate to the background thread, we introduce a shared list Backburner that contains queries that are predicted to be infeasible.  After the query thread quickly predicts failure, it passes them to the background thread, which will then re-evaluate these queries using global optimization to test whether they are truly infeasible.  Pseudocode for the the query and background thread are as follows:

\begin{algorithm}
\caption{LGO-lifelong-query$(\theta,D)$}
\label{alg:QueryLifelong}
\begin{algorithmic}[1]
\State $x \gets $LGO-query-infeasible$(\theta,D)$
\If {$x$ was produced by Step 7} 
   \State Set $D \gets D \cup \{ (\theta,x) \}$
\EndIf
\If {$x = nil$ was produced by Step 8 \par
        \hskip\algorithmicindent with $PFeasible(k') < \tau_2$}
  \State Call Push(Backburner,$\theta$)
\EndIf  
\State \Return $x$
\end{algorithmic}
\end{algorithm}

\begin{algorithm}
\caption{LGO-background-task} 
\label{alg:QueryBackground}
\begin{algorithmic}[1]
\If {Backburner is empty}
  \State Sample $\theta \sim  U(\epsilon-$BoundingBox$(D))$
\Else
  \State $\theta \gets$ Pop(Backburner)
\EndIf
\State $x \gets $Global-Optimization($\theta$)
\State $D \gets D \cup \{(\theta,x)\}$
\State Go to Step 1
\end{algorithmic}
\end{algorithm}

Step 1 of LGO-background task computes the axis-aligned range of P-parameter space spanned by feasible examples in the database BoundingBox$(D) = [\theta_{1,min},\theta_{1,max}]\times \cdots \times [\theta_{r,min},\theta_{r,max}]$. and then expands it by a constant fraction, plus a small constant amount.  This box is cached and only updated when a new example is added to $D$.  New problem samples are drawn uniformly from the expanded range.

\section{Application to Inverse Kinematics}
\label{sec:Implementation}

Here we describe an application of our implementation to the classic problem of IK.  The resulting solver combines the reliability of global optimization with the speed of local optimization.  It improves upon prior techniques by automatically handling optimality criteria, kinematic redundancy, collision avoidance, and prediction of infeasible IK queries. 

The solver is highly customizable; it accepts arbitrary robots specified by Universal Robot Description Format (URDF) files, static environment geometries given as CAD models, IK constraints, and additional user-defined feasibility constraints and objective functions.  Collision detection is performed using the Proximity Query Package (PQP)~\cite{GLM96}. Cost functions and additional constraints are specified as arbitrary Python functions that accept a configuration and optional additional arguments and return a real number.  Each IK problem specifies one or more single-link IK constraints, optional joint limits, an optional array of the movable set, and any additional arguments of custom functions.

\subsection{Problem specification}

Specifically, the user provides $s\geq 0$ IK constraints parameterized by $\theta_{IK,1},\ldots,\theta_{IK,s}$, a cost function $f(q,\theta_f)$, an optional custom inequality $g(q\theta_g) \leq 0$, and joint limits $q_{min}$ and $q_{max}$.  (Note: $f$ and/or $g$ do not need to be parameterized in which case $\theta_f$ and/or $\theta_g$ are {\em nil}. )  The optimization problem to be solved is:
\begin{equation}
\begin{gathered}
\text{Given $\theta=(\theta_f,\theta_g,\theta_{IK})$,} \\
\text{Minimize over $q$ } f(q,\theta_f) \text{ s.t.}\\
g(q,\theta_{g}) \leq 0 \\
E_{IK}(q,\theta_{IK}) = 0\\
q_{min} \leq q \leq q_{max} \\
C(q) = \text{ false}
\end{gathered}
\end{equation}
where $E_{ik}$ is the error function for the IK constraints and $C(q)$ returns whether the robot has self-collision or environmental collision at $q$.  $C$ is a nondifferentiable constraint and is implemented during optimization by setting the objective function to $\infty$ when in collision.

\subsection{Database learning}
\label{sec:IKDatabaseLearning}
To generate a database, an axis-aligned range of $\theta$ is sampled uniformly at random.  To find global optima, a modified random-restart algorithm is used.  First, we use a Newton-Raphson iterative IK solver with 100 random initial configurations to find configurations that satisfy IK and joint limit constraints.  The IK solution with lowest cost, if one exists, is then used to seed an local SQP-based optimization.  The training procedure ranges from approximately 10\,ms per example for the robot in Fig.~\ref{fig:tx90pointcomparison} to 100\,ms per example for the robot in Fig.~\ref{fig:baxter}.  Our experiments find that this technique is an order of magnitude faster than na\"{i}ve random restart local optimization.

To produce a distance metric, we have experimented with both Euclidean distance metric, as well as learned Mahalanobis distances that can correct for poorly scaled problem spaces.  Our implementation can be configured to perform online metric learning using the LogDet update method of \cite{Jain2008}.

\subsection{Experiments}

\begin{figure*}[tbp]
\centering
\includegraphics[width=0.97\textwidth]{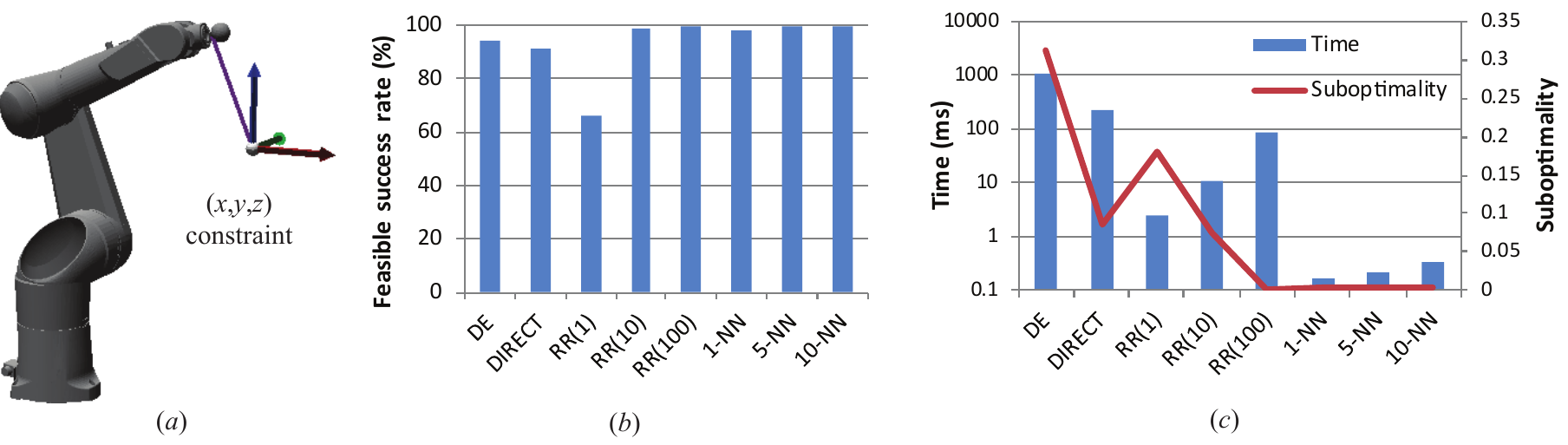}
\caption{(a) A redundant position-constrained IK problem (3 P-parameters) on a 6DOF industrial robot. (b) Comparison of success rate on a test set of 1,000 known feasible problems, between existing global optimization methods DIRECT and differential evolution (DE), the $N$-random restart method RR($N$), and $k$-NN LGO with 100,000 examples (higher values are better). (c) Comparison on average computation time, in ms, and suboptimality (lower values are better).  Time is shown on a log scale.}
\label{fig:tx90pointcomparison}
\end{figure*}

\begin{figure*}[tbp]
\centering
\includegraphics[width=0.97\textwidth]{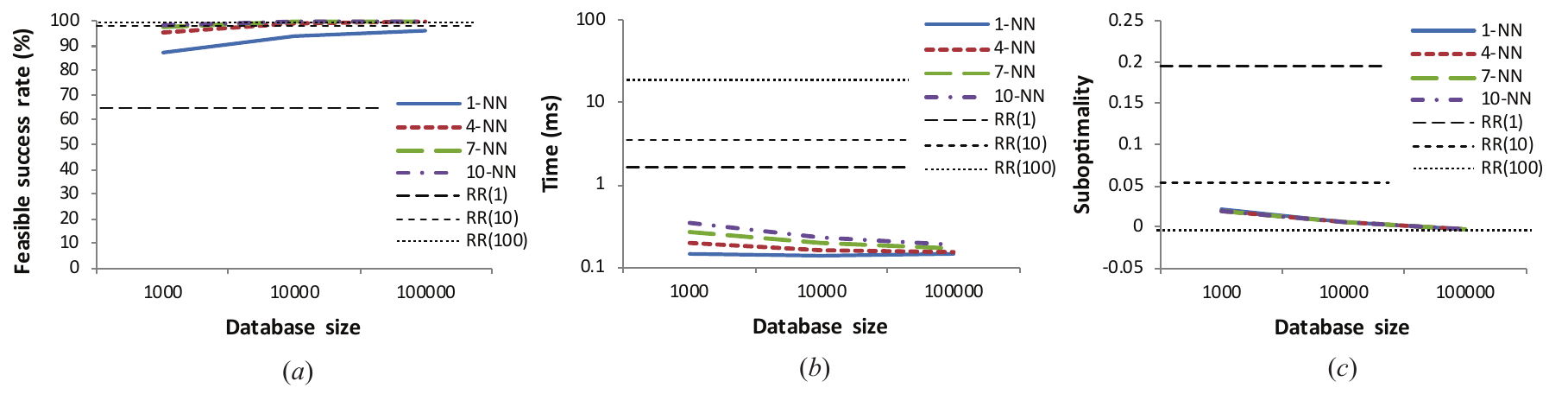}
\caption{Learning curves for LGO on the example of Fig.~\ref{fig:tx90pointcomparison} as the database size and number of neighbors $k$ varies, comparing (a) success rate (b) running time, and (c) suboptimality.  For reference, the performance of RR($N$) is shown as horizontal lines.}
\label{fig:tx90pointlearning}
\end{figure*}

\begin{figure*}[tbp]
\centering
\includegraphics[width=0.97\textwidth]{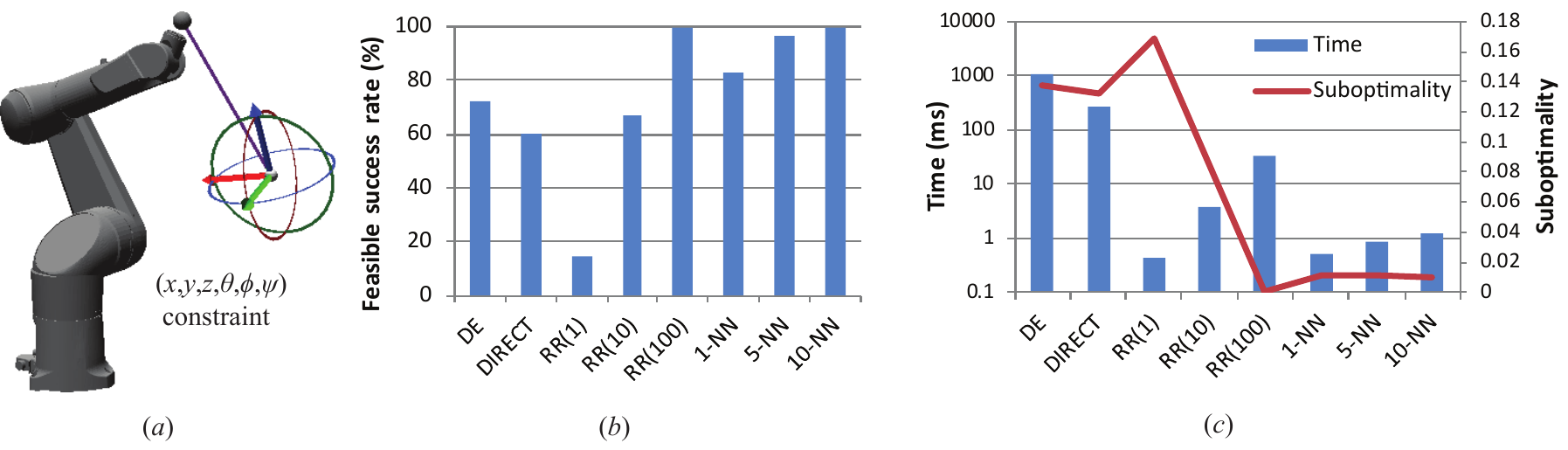}
\caption{(a) A position- and orientation-constrained IK problem (6 P-parameters) on a 6DOF industrial robot. Here LGO is tested with a 1,000,000 example database.  (b) Success rate. (c) Running time and suboptimality.}
\label{fig:tx90xformcomparison}
\end{figure*}

\begin{figure*}[tbp]
\centering
\includegraphics[width=0.97\textwidth]{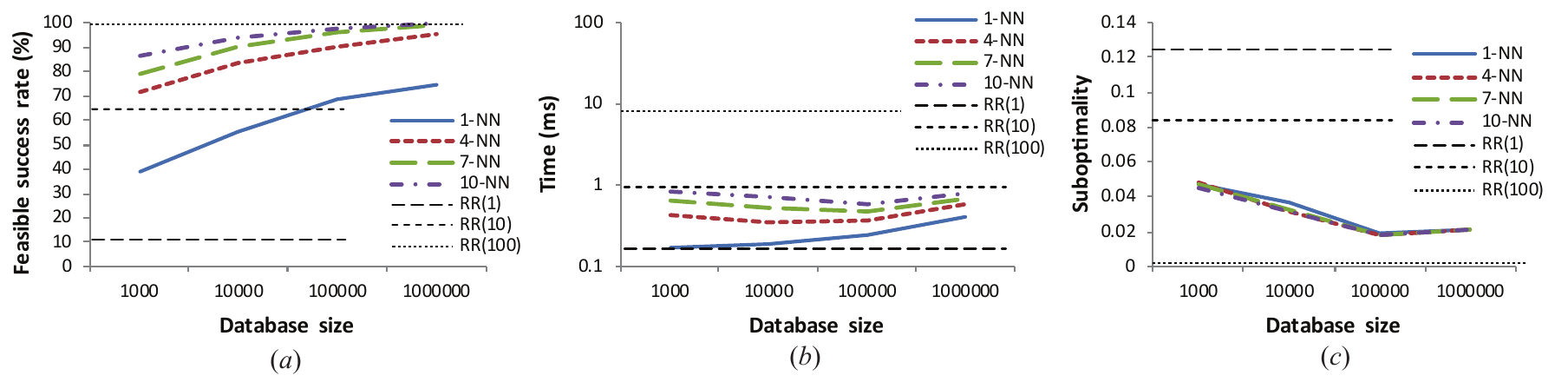}
\caption{Learning curves for LGO on the example of Fig.~\ref{fig:tx90xformcomparison} as the database size and number of neighbors $k$ varies, comparing (a) success rate (b) running time, and (c) suboptimality.}
\label{fig:tx90xformlearning}
\end{figure*}

The first experiments consider the effects of database size and selection technique for both redundant and nonredundant IK problems on an industrial robot. The cost function penalizes configurations near joint limits:
\begin{equation}
f(q) = -\sum_{i=1}^n \min(q_i-q_{min,i},q_{max,i}-q_i)^2.
\end{equation}
For each experiment, we generated a database as described in Sec.~\ref{sec:IKDatabaseLearning}, and then independently generated an test set of 1000 problems.  We disabled infeasibility prediction (i.e., set $\tau=0$) for all these experiments. 

First we consider a position-constrained problem where the 3 P-parameters of the end-effector position were varied.  Fig.~\ref{fig:tx90pointcomparison} compares the performance of the proposed LGO method, with varying numbers of neighbors $k$ and a fixed database of $|D|=100,000$ examples.  LGO is compared against the DIRECT algorithm~\cite{jones2001direct}, the metaheuristic global optimization technique differential evolution (DE), and a random-restart method RR($N$) with $N$ random restarts.  A final ``cleanup'' local optimization is run to improve the quality of the solution produced by DIRECT and DE.  The RR method is implemented as described in Sec.~\ref{sec:IKDatabaseLearning}, which is already highly tuned to this IK problem: each restart runs a Newton-Raphson technique to solve for the IK, then if successful, runs SQP.

Clearly, the off-the-shelf global optimizers are not competitive with RR($N$).  LGO outperforms even RR(1) in speed, and begins to outperform the success rate and solution quality of RR(100) at $|D|=100,000$ and $k=10$.  Compared to RR(100), LGO is two orders of magnitude faster. Fig.~\ref{fig:tx90pointlearning} illustrates the learning curves of LGO on this problem.

Fig.~\ref{fig:tx90xformcomparison} gives results for a position and orientation-constrained problem.  This is a nonredundant problem with a 6-D P-parameter space, and the robot has up to four IK solutions (elbow up/down, wrist up/down).  The IK rotation matrix is encoded via a 3D exponential map representation, and for the NN distance metric these P-parameters are scaled by $1/2\pi$.  We test LGO with training sets up to $|D|=1,000,000$ and varying values of $k$.  Again, we see that off-the-shelf methods are not competitive, and LGO is much faster than other techniques while still obtaining close to optimal results.

The learning curves, illustrated in Fig.~\ref{fig:tx90xformlearning}, show that LGO requires more data to get similar success rates to RR(100), only reaching parity at $|D|=1,000,000$. This result is consistent with the theoretical prediction that more examples are needed in higher dimensional P-parameter spaces to reach a desired level of quality.  At this point it is over 10 times faster than RR(100).  An unintuitive result is that RR($N$) is significantly faster in this case than the redundant case; the rationale is that since the problem is nonredundant, the SQP optimizer quickly terminates because it cannot make progress to improve the objective function.

Another issue is that the problem space includes axes of different units (meters for position vs. radians for orientation).  For such poorly-scaled problem spaces, we found that metric learning produced a Mahalanobis distance metric similar to our ad-hoc weighting of $1/2\pi$.  Compared to unweighted Euclidean distance, the learned metric produced a consistent, small boost in success rate (Fig.~\ref{fig:MetricLearning}).  Suboptimality and computation time were not significantly affected.

\begin{figure}[tbp]
\centering
\includegraphics[width=0.6\linewidth]{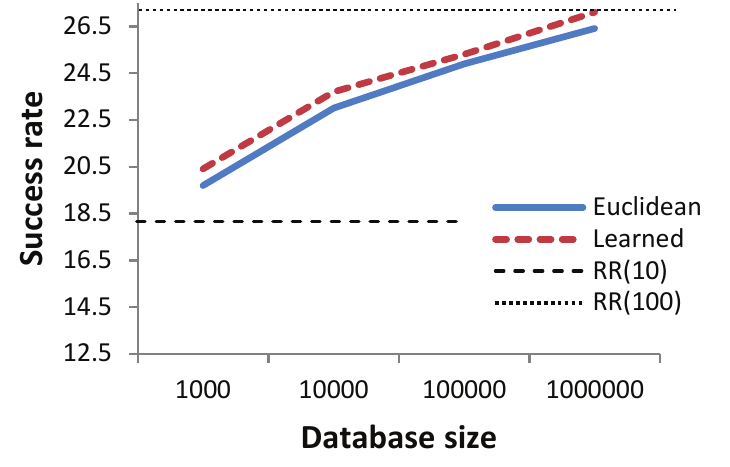}
\caption{On the problem of Fig.~\ref{fig:tx90xformcomparison}, the use of metric learning produced approximately a 2--3\% boost in success rate compared to Euclidean distance, using $k=4$ neighbors.}
\label{fig:MetricLearning}
\end{figure}

\subsection{Experiment}

\begin{figure}
\centering
\includegraphics[width=0.6\linewidth]{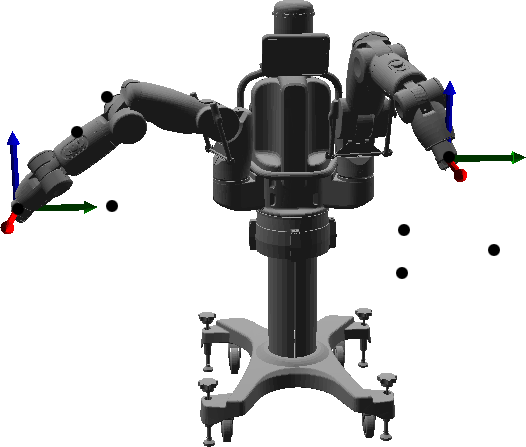}
\includegraphics[width=0.45\linewidth]{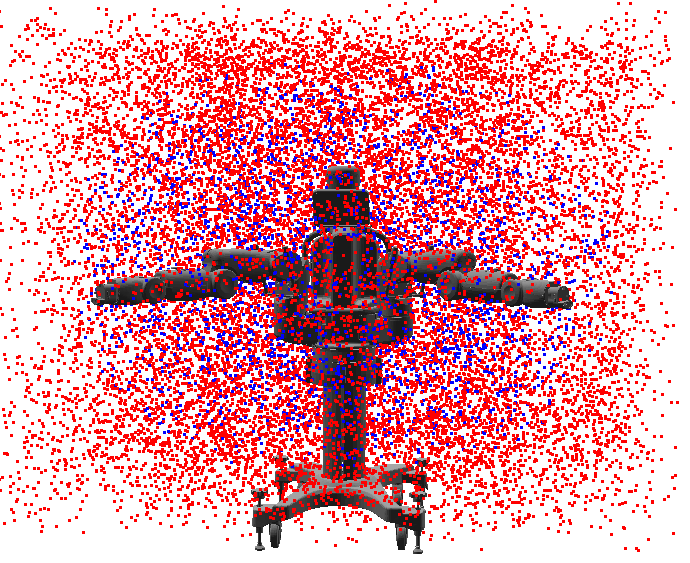}
\includegraphics[width=0.45\linewidth]{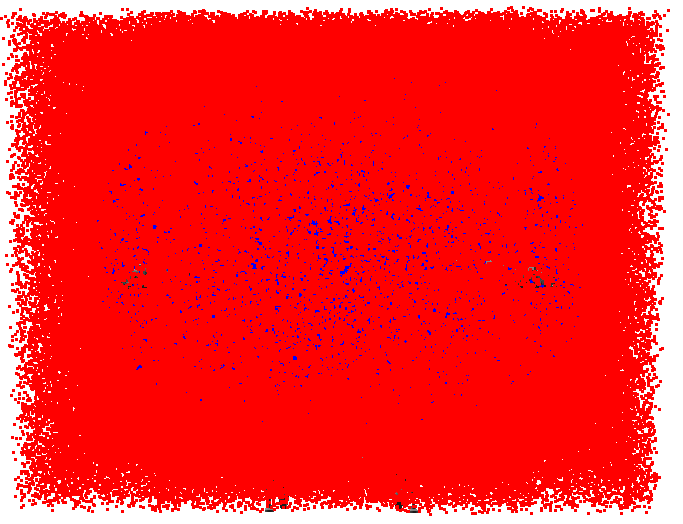}
\caption{The automatic database example with a 15-DOF robot and dual-arm position constraints.  Four seed IK queries are marked.  The IK endpoints of the database after 1h of background computation at 30\% CPU usage (approximately 10,000 examples).  The IK endpoints after 24h of computation (approximately 250,000 examples).}
\label{fig:baxter}
\end{figure}

\begin{figure}
\centering
\includegraphics[width=0.75\linewidth]{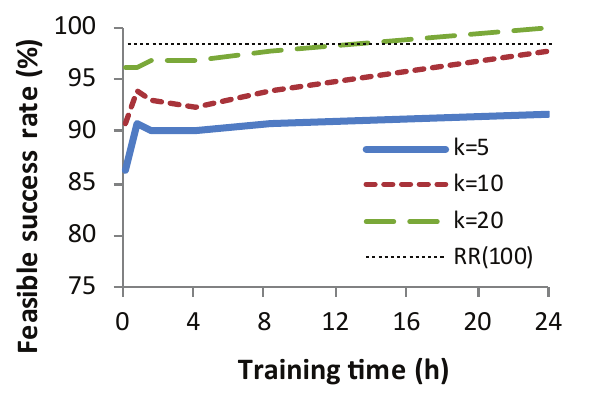}
\caption{Learning curves for the automated learning procedure on the example of~Fig.~\ref{fig:baxter}.}
\label{fig:baxter-learning}
\end{figure}

Fig.~\ref{fig:baxter} shows an example of the automatic learning procedure run on a problem with the Rethink Robotics Baxter robot with two position-constrained end effectors, for a 6D P-parameter space.  Four seed queries were provided by hand, but no other information was provided to help the system identify a feature representation or feature ranges.  It correctly identified the 6 P-parameters, and after 1 hour of background computation, the method generated a distribution of feasible problems shown in Fig.~\ref{fig:baxter}, center. The procedure was then left to learn for 24 hours, yielding approximately a quarter million examples.  Fig.~\ref{fig:baxter-learning} shows the learning curve for this training process, evaluated on a holdout testing set of 1,000 examples, 130 of which are feasible.  The resulting LGO method with $k=20$ performed more reliably than RR(100), with slightly higher solution quality, and 20 times faster (approximately 5\,ms per query compared to 100\,ms).

\section{Conclusion}

We presented an experience-driven framework for global optimization in families of related problems. Our work is an attempt to answer some fundamental questions relating optimization problem structure to the number of examples needed to attain a given quality.  First, we highlight the fact that the problem-optimum map is in general only piecewise continuous, which motivates the use of $k$-nearest neighbors approaches rather than function approximators in our implementation. Our results suggest that  the approach is practically sound in that a finite number of examples is needed to achieve a bounded level of ``goodness''.  However, the required database size depends exponentially on the problem dimensionality in the worst case.  Nevertheless, in some domains it is practical to generate databases of millions of examples, and in an inverse kinematics example problem, our implementation yields 1-5 orders of magnitude faster performance than the off-the-shelf global optimization methods, with little sacrifice in quality. 

Future work may improve the analysis by uncovering properties of problem spaces that are more easily learned, such as problem families where global optima have larger basins of attraction (e.g., convex optimization problems).  It may also be possible to bound the number of possible active sets at global optima based on the order of each constraint, or other properties of the problem.  We also have not yet considered the question of practical methods for distributing examples such that each example ``covers'' a large region of problem space with good adaptations, which would allow us to maximize performance for a given database size.



\bibliographystyle{plainnat}
\bibliography{references}

\end{document}